\documentclass{article}

 \PassOptionsToPackage{round}{natbib}

\usepackage[preprint]{nips_2018}

\usepackage[utf8]{inputenc} 
\usepackage{lipsum}
\usepackage{enumitem}

\usepackage[T1]{fontenc}    
\usepackage{hyperref}       
\usepackage{algorithm}
\usepackage{wrapfig}
\usepackage{algorithmic} 
\usepackage{graphicx}
\usepackage{amsmath}
\usepackage{amsthm}
\usepackage{url}            
\usepackage{booktabs}       
\usepackage{amsfonts}       
\usepackage{nicefrac}       
\usepackage{microtype}      

\title{Leveraging the Exact Likelihood of Deep Latent Variable Models}

\author{
  Pierre-Alexandre Mattei \\
  Department of Computer Science\\
  IT University of Copenhagen\\
  \texttt{pima@itu.dk} \\
   \And
 Jes Frellsen \\
  Department of Computer Science \\
  IT University of Copenhagen \\
  \texttt{jefr@itu.dk} \\
}

\begin{document}

\newtheorem{theorem}{Theorem}
\newtheorem{proposition}{Proposition}
\newtheorem*{proposition*}{Proposition}
\newtheorem*{theorem*}{Theorem}

\maketitle

\begin{abstract}
Deep latent variable models (DLVMs) combine the approximation abilities of deep neural networks and the statistical foundations of generative models. Variational methods are commonly used for inference; however, the exact likelihood of these models has been largely overlooked. The purpose of this work is to study the general properties of this quantity and to show how they can be leveraged in practice. We focus on important inferential problems that rely on the likelihood: estimation and missing data imputation. First, we investigate maximum likelihood estimation for DLVMs: in particular, we show that most unconstrained models used for continuous data have an unbounded likelihood function. This problematic behaviour is demonstrated to be a source of mode collapse. We also show how to ensure the existence of maximum likelihood estimates, and draw useful connections with nonparametric mixture models. Finally, we describe an algorithm for missing data imputation using the exact conditional likelihood of a deep latent variable model. On several data sets, our algorithm consistently and significantly outperforms the usual imputation scheme used for DLVMs.
\end{abstract}

\section{Introduction}

Dimension reduction aims at summarizing multivariate data using a small number of features that constitute a \emph{code}. Earliest attempts rested on linear projections, leading to Hotelling's \citeyearpar{hotelling1933} \emph{principal component analysis} (PCA) that has been vastly explored and perfected over the last century \citep{jolliffe2016principal}. In recent years, the field has been vividly animated by the successes of \emph{latent variable models}, that probabilistically use the low-dimensional features to define powerful generative models. Usually, these latent variable models transform the random code into parameters of a simple \emph{output distribution}. Linear mappings and Gaussian outputs were initially considered, giving rise to factor analysis \citep{bartholomew2011} and probabilistic principal component analysis \citep{tipping1999}. Beyond the Gaussian distribution, Bernoulli or more general exponential family outputs have been considered. In recent years, much work has been done regarding nonlinear mappings parametrised by deep neural networks, following the seminal papers of \citet{rezende2014} and \citet{kingma2014}. These models have led to impressive empirical performance in unsupervised and semi-supervised generative modelling of images \citep{narayanaswamy2017}, molecular structures \citep{kusner2017,gomez2018}, or arithmetic expressions \citep{kusner2017}. This paper is an investigation of the statistical properties of these models, which remain essentially unknown.


\subsection{Deep latent variable models}

In their most common form, \emph{deep latent variable models} (DLVMs) assume that we are in presence of a data matrix $\mathbf{X} = (\mathbf{x}_1,...,\mathbf{x}_n)^T \in \mathcal{X}^n$ that we wish to explain using some latent variables $\mathbf{Z} = (\mathbf{z}_1,...,\mathbf{z}_n)^T \in \mathbb{R}^{n \times d}$. We assume that $(\mathbf{x}_i,\mathbf{z}_i)_{i\leq n}$ are independent and identically distributed (i.i.d.) random variables driven by the following generative model:
\begin{equation}
\left\{
\begin{array}{ll}
\mathbf{z} \sim p(\mathbf{z}) \\
p_{\boldsymbol{\theta}}(\mathbf{x}|\mathbf{z}) = \Phi(\mathbf{x}| f_{\boldsymbol{\theta}}(\mathbf{z})).
\end{array}
\right.
\end{equation}

The unobserved random vector $\mathbf{z} \in \mathbb{R}^d$ is called the \emph{latent variable} and usually follows marginally a simple distribution $p(\mathbf{z})$ called the \emph{prior distribution}. The dimension $d$ of the latent space is called the \emph{intrinsic dimension}, and is usually smaller than the dimensionality of the data. The collection $(\Phi (\cdot|\boldsymbol{\eta}))_{\boldsymbol{\eta} \in H}$ is a parametric family of \emph{output densities} with respect to a dominating measure (usually the Lebesgue or the counting measure). The function $f_{\boldsymbol{\theta}}: \mathbb{R}^d \rightarrow H$ is called a \emph{decoder} or a \emph{generative network}, and is parametrised by a (deep) neural network whose weights are stored in $\boldsymbol{\theta} \in \boldsymbol{\Theta}$. This parametrisation allows to leverage recent advances in deep architectures, such as deep residual networks \citep{kingma2016}, recurrent networks \cite{bowman2016}, or batch normalisation \citep{sonderby2016}. Several output families have been considered:
\begin{itemize}[leftmargin=3ex,topsep=0ex]
	\item In case of discrete multivariate data, products of Multinomial distributions.
	\item In case of real multivariate data, multivariate Gaussian distributions.
	\item When dealing with images, several specific proposals have been made (e.g. the discretised logistic mixture approach of \citealp{salimans2017}). 
	\item Dirac outputs correspond to \emph{deterministic decoders}, that are used e.g. within generative adversarial networks (\citealp{goodfellow2014}), or non-volume preserving transformations \citep{dinh2017}.
\end{itemize}

The Gaussian and Bernoulli families are the most widely studied, and will be the focus of this article. The latent structure of these DLVMs leads to the following marginal distribution of the data:
\begin{equation}
p_{\boldsymbol{\theta}}(\mathbf{x})=\int_{\mathbb{R}^d} p_{\boldsymbol{\theta}}(\mathbf{x}|\mathbf{z}) p(\mathbf{z}) d\mathbf{z} = \int_{\mathbb{R}^d} \Phi(\mathbf{x}| f_{\boldsymbol{\theta}}(\mathbf{z})) p(\mathbf{z}) d\mathbf{z}.
\end{equation}

\subsection{Scalable learning through amortised variational inference}

The log-likelihood function of a DLVM is, for all $\boldsymbol{\theta} \in \boldsymbol{\Theta}$,
\begin{equation}
\ell({\boldsymbol{\theta}})= \log p_{\boldsymbol{\theta}}(\mathbf{X}) =\sum_{i=1}^n \log p_{\boldsymbol{\theta}}(\mathbf{x}_i),
\end{equation}
which is an extremely challenging quantity to compute that involves potentially high-dimensional integrals. Estimating $\boldsymbol{\theta}$ by maximum likelihood appears therefore out of reach. Consequently, following \citet{rezende2014} and \citet{kingma2014}, inference in DLVMs is usually performed using \emph{amortised variational inference}. Variational inference approximatively maximises the log-likelihood by maximising a lower bound known as the \emph{evidence lower bound} (ELBO, see e.g. \citealp{blei2017}):
\begin{equation}
\textup{ELBO}(\boldsymbol{\theta}, q) = \mathbb{E}_{\mathbf{Z}\sim q} \left[ \log \frac{p(\mathbf{X},\mathbf{Z})}{q(\mathbf{Z})}\right] =  \ell({\boldsymbol{\theta}}) - \textup{KL}(q||p(\cdot | \mathbf{X})  ) 
\leq \ell(\boldsymbol{\theta}).
\end{equation}
where the \emph{variational distribution} $q$ is a distribution over the space of codes $\mathbb{R}^{n\times d}$. The variational distribution plays the role of a tractable approximation of the posterior distribution of the codes; when this approximation is perfectly accurate, the ELBO is equal to the log-likelihood. \emph{Amortised inference} builds $q$ using a neural network called the \emph{inference network} $g_{\boldsymbol{\gamma}}: \mathcal{X} \rightarrow K$, whose weights are stored in $\boldsymbol{\gamma} \in \mathbf{\Gamma}$:
\begin{equation}q_{\boldsymbol{\gamma},\mathbf{X}}(\mathbf{Z}) = q_{\boldsymbol{\gamma},\mathbf{X}}(\mathbf{z}_1,...,\mathbf{z}_n) = \prod_{i=1}^n \Psi(\mathbf{z}_i | g_{\boldsymbol{\gamma}} (\mathbf{x}_i)),\end{equation}
where $(\Psi(\cdot | \boldsymbol{\kappa}))_{\boldsymbol{\kappa} \in K}$ is a parametric family of distributions over $\mathbb{R}^d$ -- such as Gaussians with diagonal covariances \citep{kingma2014}. Other kinds of families -- built using e.g. normalising flows \citep{rezende2015,kingma2016}, auxiliary variables \citep{maaloe2016,ranganath2016}, or importance weights \citep{burda2016,cremer2017} -- have been considered for amortised inference, but they will not be central focus of in this paper. Within this framework,  variational inference for DLVMs solves the optimisation problem
$\max_{\boldsymbol{\theta}\in \boldsymbol{\Theta},\boldsymbol{\gamma} \in \boldsymbol{\Gamma}} \textup{ELBO}(\boldsymbol{\theta}, q_{\boldsymbol{\gamma},\mathbf{X}})$
using variants of stochastic gradient ascent (see e.g. \citealp{roeder2017}, for strategies for computing gradients estimates of the ELBO).

As emphasised by \citet{kingma2014}, the ELBO resembles the objective function of a popular deep learning model called an \emph{autoencoder} (see e.g. \citealp{goodfellow2016}, Chapter 14). This motivates the popular denomination of \emph{encoder} for the inference network $g_{\boldsymbol{\gamma}}$ and \emph{variational autoencoder} (VAE) for the combination of a DLVM with amortised variational inference.

\paragraph{Contributions.} In this work, we revisit DLVMs
by asking: \emph{Is it possible to leverage the properties of $p_{\boldsymbol{\theta}}(\mathbf{x})$ to understand and improve deep generative modelling?}
Our main contributions are:
\begin{itemize}[leftmargin=3ex,topsep=0ex]
	\item We show that \emph{maximum likelihood is ill-posed for continuous DLVMs and well-posed for discrete ones}. We link this undesirable property of continuous DLVMs to the mode collapse phenomenon, and illustrate it on a real data set.
	\item We draw a connection between DLVMs and nonparametric statistics, and show that \emph{DLVMs can be seen as parsimonious submodels of nonparametric mixture models}. 
	\item We leverage this connection to provide a way of finding an \emph{upper bound of the likelihood based on finite mixtures}. Combined with the ELBO, this bound allows us to provide useful ``sandwichings'' of the exact likelihood. We also prove that this bound characterises the large capacity behaviour of DLVMs.
	\item When dealing with missing data, we show how a simple modification of an approximate scheme proposed by \citet{rezende2014} allows us to \emph{draw according to the exact conditional distribution of the missing data}. On several data sets and missing data scenarios, our algorithm consistently outperforms the one of \citet{rezende2014}, while having the same computational cost.
\end{itemize}

\section{Is maximum likelihood well-defined for deep latent variable models?}

In this section, we investigate the properties of maximum likelihood estimation for DLVMs with Gaussian and Bernoulli outputs.

\subsection{On the boundedness of the likelihood of deep latent variable models}

Deep generative models with Gaussian outputs assume that the data space is $\mathcal{X}=\mathbb{R}^p$,  and that the family of output distributions is the family of $p$-variate full-rank Gaussian distributions. The conditional distribution of each data point is consequently
\begin{equation}
p_{{\boldsymbol{\theta}}}(\mathbf{x}|\mathbf{z}) = \mathcal{N}(\mathbf{x}|{\boldsymbol{\mu}_{\boldsymbol{\theta}}}(\mathbf{z}),{\boldsymbol{\Sigma}_{\boldsymbol{\theta}}}(\mathbf{z})),
\end{equation}
where ${\boldsymbol{\mu}_{\boldsymbol{\theta}}} : \mathbb{R}^d \rightarrow \mathbb{R}^p$ and ${\boldsymbol{\Sigma}_{\boldsymbol{\theta}}} : \mathbb{R}^d \rightarrow \mathcal{S}_p^{++}$ are two continuous functions parametrised by neural networks whose weights are stored in a parameter $\boldsymbol{\theta}$. These two functions constitute the decoder of the model. This leads to the log-likelihood
\begin{equation}
\ell({\boldsymbol{\theta}})= \sum_{i=1}^n \log \left( \int_{\mathbb{R}^d} \mathcal{N}(\mathbf{x}_i|{\boldsymbol{\mu}_{\boldsymbol{\theta}}}(\mathbf{z}),{\boldsymbol{\Sigma}_{\boldsymbol{\theta}}}(\mathbf{z})) p(\mathbf{z})d\mathbf{z} \right).
\end{equation}
This model can be seen as a special case of \emph{infinite mixture of Gaussian distributions}. However, it is well-known that maximum likelihood is ill-posed for \emph{finite} Gaussian mixtures (see e.g. \citealp{lecam1990}). Indeed, the likelihood function is unbounded above, and its infinite maxima happen to be very poor generative models. This problematic behaviour of a model quite similar to DLVMs motivates the question: \emph{is the likelihood function of DLVMs bounded above?}

In this section, we will not make any particular parametric assumption about the prior distribution of the latent variable $\mathbf{z}$. Both simple isotropic Gaussian distributions and more complex learnable priors have been proposed, but all of them are affected by the negative results of this section. We simply make the natural assumptions that $\mathbf{z}$ is continuous and has zero mean. Many different neural architectures have been explored regarding the parametrisation of the decoder. For example, \citet{kingma2014} consider multilayer perceptrons (MLPs) of the form
\begin{equation} \label{eq:K1} {\boldsymbol{\mu}_{\boldsymbol{\theta}}}(\mathbf{z}) = \mathbf{V} \tanh \left(\mathbf{W}\mathbf{z} +  \mathbf{a} \right)   +  \mathbf{b}, \; {\boldsymbol{\Sigma}_{\boldsymbol{\theta}}}(\mathbf{z})  = \textup{Diag} \left( \exp \left({\boldsymbol{\alpha}} \tanh \left(\mathbf{W}\mathbf{z} +  \mathbf{a} \right)   + \boldsymbol{\beta} \right)\right),\end{equation}
where $\boldsymbol{\theta} = (\mathbf{W},\mathbf{a},\mathbf{V}, \mathbf{b},\boldsymbol{\alpha},\beta)$. The weights of the decoder are $\mathbf{W} \in \mathbb{R}^{h \times d},\mathbf{a} \in \mathbb{R}^{h},\mathbf{V},\boldsymbol{\alpha} \in \mathbb{R}^{p \times h}$, and $\mathbf{b}, \boldsymbol{\beta} \in  \mathbb{R}^p$. The integer $h \in \mathbb{N}^*$ is the (common) number of \emph{hidden units} of the MLPs. Much more complex parametrisations exist, but we will see that this one, arguably one of the most rigid, is already too flexible for maximum likelihood. Actually, we will show that an even much less flexible family of MLPs with a single hidden unit is problematic. Let $\mathbf{w} \in \mathbb{R}^p$ and let $(\alpha_k)_{k\geq 1}$ be a sequence of nonnegative real numbers such that $\alpha_k \rightarrow +\infty$ as $k \rightarrow +\infty$. For all $i \leq n$, with the parametrisation \eqref{eq:K1}, we consider the sequences of parameters $\boldsymbol{\theta}^{(i,\mathbf{w})}_{k} = (\alpha_k\mathbf{w}^T,-\alpha_k,0,0,\mathbf{x}_i,\alpha_k\mathbf{1}_p,-\alpha_k\mathbf{1}_p)$. This leads to the following simplified decoders:
\begin{equation} {\boldsymbol{\mu}_{\boldsymbol{\theta}^{(i,\mathbf{w})}_{k}}}(\mathbf{z}) = \mathbf{x}_i, \; {\boldsymbol{\Sigma}_{\boldsymbol{\theta}^{(i,\mathbf{w})}_{k}}}(\mathbf{z})  = \exp \left( \alpha_k \tanh \left(\alpha_k \mathbf{w}^T\mathbf{z}\right) - \alpha_k \right)\mathbf{I}_p.
\end{equation}
As shown by next theorem, these sequences of decoders lead to the divergence of the log-likelihood function.

\begin{theorem} For all $i \in \{1,...,n\}$ and $\mathbf{w} \in \mathbb{R}^d \setminus \{0\}$, we have
	$\lim_{k \rightarrow + \infty} \ell \left(\boldsymbol{\theta}^{(i,\mathbf{w})}_{k}\right) = + \infty. $
\end{theorem}
\begin{proof}
A detailed proof is provided in Appendix A. Its cornertone is the fact that the sequence of  functions ${\boldsymbol{\Sigma}_{\boldsymbol{\theta}^{(i,\mathbf{w})}_{k}}}$ converges to a function that outputs both singular and nonsingular covariances, leading to the explosion of $\log p_{\boldsymbol{\theta}_k}(\mathbf{x}_i)$ while all other terms in the likelihood remain bounded below by a constant.
\end{proof}

Using simple MLP-based parametrisations such a the one of \citet{kingma2014} therefore brings about an unbounded log-likelihood function. A natural question that follows is: do these infinite suprema lead to useful generative models? The answer is no. Actually, none of the functions considered in Theorem 1 are particularly useful, because of the use of a constant mean function. This is formalised in the next proposition, that exhibits a strong link between likelihood blow-up and the \emph{mode collapse} phenomenon.

\begin{proposition}
	For all $k \in \mathbb{N}^*$, $i \in \{1,...,n\}$, and $\mathbf{w} \in \mathbb{R}^d \setminus \{0\}$, the distribution $p_{\boldsymbol{\theta}^{(i,\mathbf{w})}_{k}}$ is spherically symmetric and unimodal around $\mathbf{x}_i$.
\end{proposition}
\begin{proof}
Because of the constant mean function and the isotropic covariance, the density of $\mathbf{x}$ is a decreasing function $||\mathbf{x}-\mathbf{x}_i||_2$, hence the spherical symmetry and the unimodality. Note that there are several different definitions for multivariate unimodality (see e.g. \citealp{dharmadhikari1988}). Here, we mean that the only local maximum of the density is at $\mathbf{x}_i$.\end{proof}

The spherical symmetry implies that the distribution of these ``optimal'' deep generative model will lead to uncorrelated variables, and the unimodality will lead to poor sample diversity. This behaviour is symptomatic of mode collapse, which remains one of the most challenging drawbacks of generative modelling \citep{arora2018}. Unregularised gradient-based optimisation of a tight lower bound of this unbounded likelihood is therefore likely to chase these (infinitely many) bad suprema. This gives a theoretical foundation to the necessary regularisation of VAEs that was already noted by \citet{rezende2014} and \citet{kingma2014}. For example, using weight decay as in \citet{kingma2014} is likely to help avoiding these infinite maxima. This difficulty to learn the variance may also explain the choice made by several authors to use a constant variance function of the form $ \boldsymbol{\Sigma}( \mathbf{z}) = \sigma_0 \mathbf{I}_p$, where $\sigma_0$ can be either fixed \citep{zhao2017} or learned via approximate maximum likelihood \citep{pu2016}.

\paragraph{Tackling the unboundedness of the likelihood.}
Let us go back to a parametrisation which is not necessarily MLP-based. Even in this general context, it is possible to tackle the unboundedness of the likelihood using additional constraints on ${\boldsymbol{\Sigma}_{\boldsymbol{\theta}}}$. Specifically, for each $\xi \geq 0$, we will consider the set
$
\mathcal{S}_p^\xi = \{\mathbf{A} \in \mathcal{S}_p^+ | \min (\textup{Sp}\mathbf{A}) \geq \xi \}.
$
Note that $\mathcal{S}_p^0=\mathcal{S}_p^+$. This simple spectral constraint allows to end up with a bounded likelihood. 
\begin{proposition} \label{prop:constraints}
	Let $\xi>0$. If the parametrisation of the decoder is such that the image of ${\boldsymbol{\Sigma}_{\boldsymbol{\theta}}}$ is included in $\mathcal{S}_p^\xi$ for all $\boldsymbol{\theta}$, then the log-likelihood function is upper bounded by $-np\log \sqrt{2\pi \xi}$.
\end{proposition}
\begin{proof}
	For all $i \in \{1,...,n\}$, we have
	\begin{equation*}
	p(\mathbf{x}_i|{\boldsymbol{\mu}_{\boldsymbol{\theta}}},{\boldsymbol{\Sigma}_{\boldsymbol{\theta}}}) =\int_{\mathbb{R}^d} \mathcal{N}(\mathbf{x}|{\boldsymbol{\mu}_{\boldsymbol{\theta}}}(\mathbf{z}),{\boldsymbol{\Sigma}_{\boldsymbol{\theta}}}(\mathbf{z})) \mathcal{N}(\mathbf{z}|0,\mathbf{I}_d) d \mathbf{z} \leq \frac{1}{(2\pi \xi)^{p/2}},
	\end{equation*}
	using the fact that the determinant of ${\boldsymbol{\Sigma}_{\boldsymbol{\theta}}}(\mathbf{z})$ is lower bounded by $\xi^p$ for all $\mathbf{z} \in \mathbb{R}^d$ and that the exponential of a negative number is smaller than one. Therefore, the likelihood function is bounded above by $1/(2\pi \xi)^{np/2}$.
\end{proof}
Similar constraints have been proposed to solve the ill-posedness of maximum likelihood for finite Gaussian mixtures (e.g. \citealp{hathaway1985,biernacki2011}). In practice, implementing such constraints can be easily done by adding a constant diagonal matrix to the output of the covariance decoder.

\paragraph{What about other parametrisations?} We chose a specific and natural parametrisation in order to obtain a constructive proof of the unboundedness of the likelihood. However, virtually any other deep parametrisation that does not include covariance constraints will be affected by our result, because of the universal approximation abilities of neural networks (see e.g. \citealp[Section 6.4.1]{goodfellow2016}).

\paragraph{Bernoulli DLVMs do not suffer from unbounded likelihood.}
When $\mathcal{X}=\{0,1\}^p$, Bernoulli DLVMs assume that $(\Phi (\cdot|\boldsymbol{\eta}))_{\boldsymbol{\eta}\in H}$ is the family of $p$-variate multivariate Bernoulli distributions (i.e. the family of products of $p$ univariate Bernoulli distributions). In this case, maximum likelihood is well-posed.

\begin{proposition}
	Given any possible parametrisation, the log-likelihood function of a deep latent model with Bernoulli outputs is everywhere negative.
\end{proposition}
\begin{proof}
	This is a direct consequence of the fact that the density of a Bernoulli distribution is always smaller than one.
\end{proof}

\subsection{Towards data-dependent likelihood upper bounds}

We have determined under which conditions maximum likelihood estimates exist, and have computed simple upper bounds on the likelihood functions. Since they do not depend on the data, these bounds are likely to be very loose. A natural follow-up issue is to seek tighter, data-dependent upper bounds that remain easily computable. Such bounds are desirable because, combined with ELBOs, they would allow sandwiching the likelihood between two bounds.

To study this problem, let us take a step backwards and consider a more general infinite mixture model. Precisely, given any distribution $G$ over the generic parameter space $H$, we define the \emph{nonparametric mixture model} (see e.g. \citealp[Chapter 1]{lindsay1995}) as:
\begin{equation}
p_G(\mathbf{x}) =   \int_{H} \Phi (\mathbf{x}|\boldsymbol{\eta})dG(\boldsymbol{\eta}).
\end{equation}
Note that there are many ways for a mixture model to be nonparametric (e.g. having some nonparametric components, an infinite but countable number of components,  or un uncountable number of components). In this case, this comes from the fact that the model parameter is the \emph{mixing distribution} $G$, which belongs to the set $\mathcal{P}$ of all probability measures over $H$. The likelihood of any $G \in \mathcal{P}$ is given by
\begin{equation}
\ell(G) =  \sum_{i=1}^n\log p_G(\mathbf{x}_i).
\end{equation}

When $G$ has a finite support of cardinal $k \in \mathbb{N}^*$, $p_G$ is a finite mixture model with $k$ components. When the mixing distribution $G$ is generatively defined by the distribution of a random variable $\boldsymbol{\eta}$ such that 
$
\mathbf{z} \sim p(\mathbf{z}), \; \boldsymbol{\eta} = f_{\boldsymbol{\theta}}(\mathbf{z}),
$
we exactly end up with a deep generative model with decoder $f_{\boldsymbol{\theta}}$. Therefore, the nonparametric mixture is a more general model that the DLVM. The fact that the mixing distribution of a DLVM is intrinsically low-dimensional leads us to interpret the DLVM as a \emph{parsimonious submodel of the nonparametric mixture model}. This also gives us an immediate upper bound on the likelihood of any decoder $f_{\boldsymbol{\theta}}$:
$
\ell({\boldsymbol{\theta}}) \leq \max_{G \in \mathcal{P}}\ell (G).
$

Of course, in many cases, this upper bound will be infinite (for example in the case of unconstrained Gaussian outputs). However, under the conditions of boundedness of the likelihood of deep Gaussian models, the bound is finite and attained for a finite mixture model with at most $n$ components.
\begin{theorem}
	Assume that $(\Phi (\cdot|\boldsymbol{\eta}))_{\boldsymbol{\eta}\in H}$ is the family of multivariate Bernoulli distributions or the family of Gaussian distributions with the spectral constraint of Proposition \ref{prop:constraints}. The likelihood of the corresponding nonparametric mixture model is maximised for a finite mixture model of $k \leq n$ distributions from the family $(\Phi (\cdot|\boldsymbol{\eta}))_{\boldsymbol{\eta}\in H}$.
\end{theorem}
\begin{proof}
A detailed proof is provided in Appendix B. The main tool of the proof of this rather surprising result is Lindsay's \citeyearpar{lindsay1983} geometric analysis of the likelihood of nonparametric mixtures, based on Minkovski's theorem. Specifically, Lindsay's \citeyearpar{lindsay1983} Theorem 3.1 ensures that, when the trace of the likelihood curve is compact, the likelihood function is maximised for a finite mixture. For the Bernoulli case, compactness of the curve is immediate; for the Gaussian case, we use a compactification argument inspired by \cite{van1992}.
\end{proof}

\begin{wrapfigure}{r}{0.50\textwidth}
	\vspace{-3mm}
	\centering
	\includegraphics[width=0.50\columnwidth]{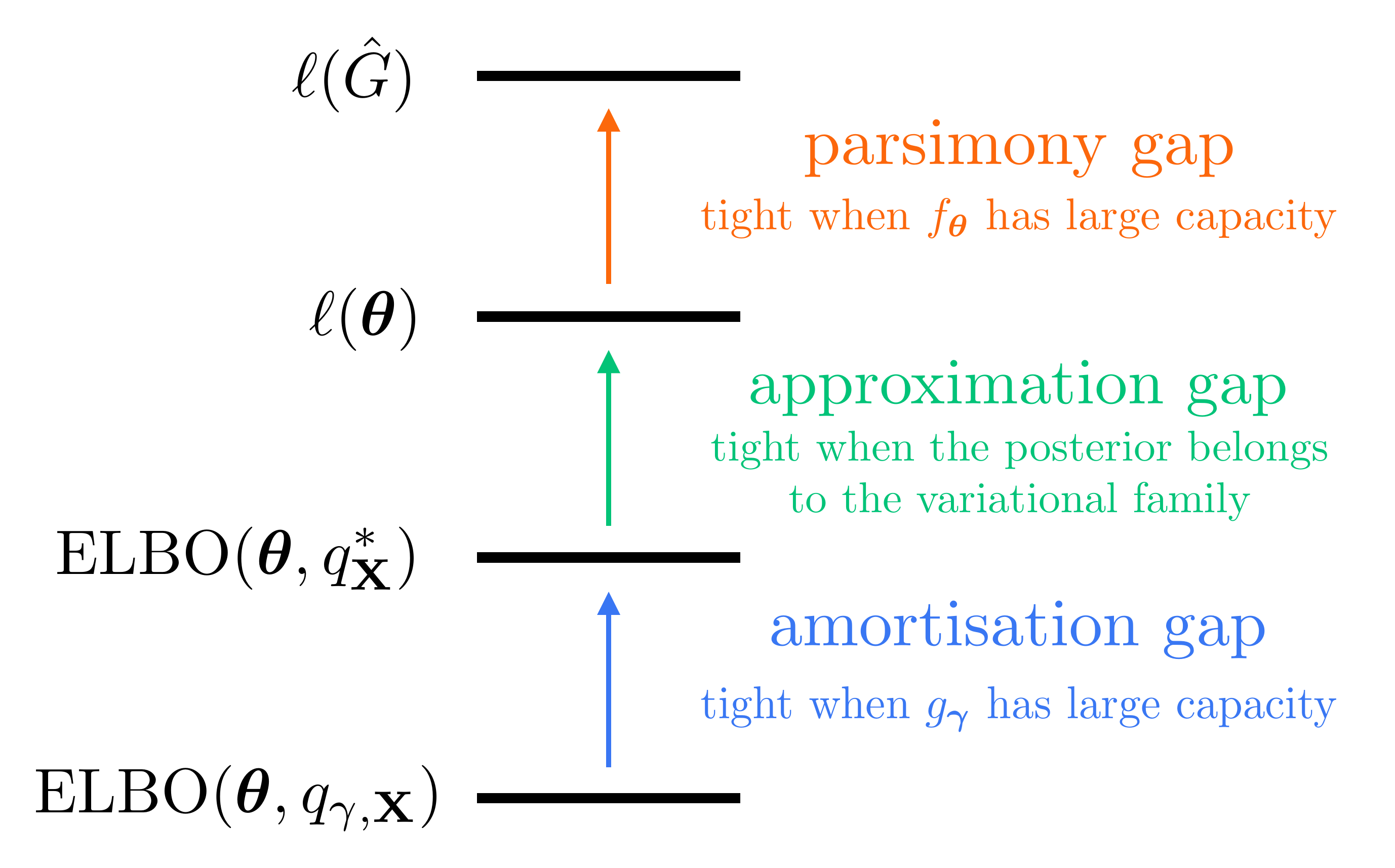}
	\caption{The parsimony gap represents the amount of likelihood lost due to the architecture of the decoder. The approximation gap expresses how far the posterior is from the variational family, and the amortisation gap appears due to the limited capacity of the encoder \citep{cremer2018}.}
	\label{fig:gaps}
\end{wrapfigure}

Assume now that the conditions of Theorem 2 are satisfied. Let us denote a maximum likelihood estimate of $G$ as $\hat{G}$.
For all $\boldsymbol{\theta}$, we therefore have \begin{equation}
\ell(\boldsymbol{\theta}) \leq \ell (\hat{G}),
\end{equation}
which gives an \emph{upper bound} on the likelihood. We call the difference $\ell (\hat{G}) - \ell(\boldsymbol{\theta}) $ the \emph{parsimony gap} (see Fig. \ref{fig:gaps}). By sandwiching the exact likelihood between this bound and an ELBO, we can also have guarantees on how far a posterior approximation $q$ is from the true posterior:
\begin{equation}
\textup{KL}(q||p(\cdot | \mathbf{X}))  \leq \ell (\hat{G}) - \textup{ELBO}(\boldsymbol{\theta},q).
\end{equation}

From a computational perspective, the estimate $\hat{G}$ can be found using the expectation-maximisation algorithm for finite mixtures \citep{dempster1977} -- although it only ensures to find a local optimum. Some strategies guaranteed to find a global optimum have also been developed (e.g. \citealp[Chapter 6]{lindsay1995}, or \citealp{wang2007}).


Now that computationally approachable upper bounds have been derived, the question remains whether or not these bounds can be tight. Actually, as shown by next theorem, tightness of the parsimony gap occurs when the decoder has universal approximation abilities. In other words, \emph{the nonparametric upper bound characterises the large capacity limit of the decoder.}

\begin{theorem}[\textbf{Tightness of the parsimony gap}]
	\label{th:nonpara}
	Assume that
	
		1. $(\Phi (\cdot|\boldsymbol{\eta}))_{\boldsymbol{\eta}\in H}$ is the family of multivariate Bernoulli distributions or the family of Gaussian distributions with the spectral constraint of Proposition \ref{prop:constraints},
		
		2. The decoder has universal approximation abilities: for any compact $C \subset  \mathbb{R}^d$ and continuous function $f: C \rightarrow H$, for all $\varepsilon >0$, there exists $\boldsymbol{\theta} \in \boldsymbol{\Theta}$ such that $||f - f_{\boldsymbol{\theta}} ||_\infty < \varepsilon$.
	
	Then, for all $\varepsilon >0$, there exists $\boldsymbol{\theta} \in \boldsymbol{\Theta}$ such that
	$    \ell(\hat{G}) \geq \ell(\boldsymbol{\theta}) \geq \ell(\hat{G}) - \varepsilon$.
\end{theorem}
\begin{proof}
A detailed proof is provided in Appendix C. The main idea is to split the code space into a compact set made of several parts that will represent the mixture components, and an unbounded set of very small prior mass. The universal approximation property is finally used for this compact set.
\end{proof}

The universal approximation condition is satisfied for example by MLPs with nonpolynomial activations \citep{leshno1993}.
Combined with the work of \citet{cremer2018}, who studied the large capacity limit of the \emph{encoder}, this result describes the general behaviour of a VAE in the large capacity limit (see Fig. \ref{fig:gaps}).

\section{Missing data imputation using the exact conditional likelihood}

In this section, we assume that a variational autoencoder  has been trained, and that some data is missing at test time. The couple decoder/encoder obtained after training is denoted by $f_{\boldsymbol{\theta}}$ and $g_{\boldsymbol{\gamma}}$. Let $\mathbf{x} \in \mathcal{X}$ be a new data point that consists in some observed features $\mathbf{x}^\textup{obs}$ and missing data $\mathbf{x}^\textup{miss}$.
Since we have a probabilistic model $p_{\boldsymbol{\theta}}$ of the data, an ideal way of imputing $\mathbf{x}^\textup{miss}$ would be to generate some data according to the \emph{conditional likelihood}
\begin{equation}
p_{\boldsymbol{\theta}}(\mathbf{x}^\textup{miss} | \mathbf{x}^\textup{obs}) = \int_{\mathbb{R}^d} p_{\boldsymbol{\theta}}(\mathbf{x}^\textup{miss} |  \mathbf{x}^\textup{obs},\mathbf{z}) p(\mathbf{z}| \mathbf{x}^\textup{obs})d\mathbf{z}.\end{equation}
Again, this distribution appears out of reach because of the integration of the latent variable $\mathbf{z}$. However, it is reasonable to assume that, for all $\mathbf{z} \in \mathbb{R}^d$,
$p_{\boldsymbol{\theta}}(\mathbf{x}^\textup{miss} |  \mathbf{x}^\textup{obs},\mathbf{z})$
is easily computable (this is the case for Bernoulli or Gaussian outputs). Under this assumption, we will see that generating data according to the conditional distribution is actually possible.

\subsection{Pseudo-Gibbs sampling}

\citet{rezende2014} proposed a simple way of imputing $\mathbf{x}^\textup{miss}$ by following a Markov chain $(\mathbf{z}_t,\hat{\mathbf{x}}^\textup{miss}_{t})_{t\geq 1}$(initialised by randomly imputing the missing data with $\hat{\mathbf{x}}^\textup{miss}_0$). For all $t\geq 1$, the chain alternatively generates $\mathbf{z}_t \sim \Psi(\mathbf{z}|g_{\boldsymbol{\gamma}}(\mathbf{x}^\textup{obs},\hat{\mathbf{x}}^\textup{miss}_{t-1}))$ and $ \hat{\mathbf{x}}^\textup{miss}_{t} \sim \Phi(\mathbf{x}^\textup{miss}|\mathbf{x}^\textup{obs}, f_{\boldsymbol{\theta}}(\mathbf{z}_{t}))$ until convergence.
This scheme closely resembles Gibbs sampling \citep{Geman1984}, and actually exactly coincides with Gibbs sampling when the amortised variational distribution $\Psi(\mathbf{z}|g_{\boldsymbol{\gamma}}(\mathbf{x}^\textup{obs},\hat{\mathbf{x}}^\textup{miss}))$ is equal to the true posterior distribution $p_{\boldsymbol{\theta}}(\mathbf{z}|\mathbf{x}^\textup{obs},\hat{\mathbf{x}}^\textup{miss})$ for all possible $\hat{\mathbf{x}}^\textup{miss}$. Following the terminology of \citet{heckerman2000}, we will call this algorithm \emph{pseudo-Gibbs sampling}. Very similar schemes have been proposed for more general autoencoder settings \citep[Section 20.11]{goodfellow2016}. Because of its flexibility, this pseudo-Gibbs approach is routinely used for missing data imputation using DLVMs (see e.g. \citealp{li2016,rezende2016,du2018}).
\citet[Proposition F.1]{rezende2014} proved that, when these two distributions are close in some sense, pseudo-Gibbs sampling generates points that approximatively follow the conditional distribution $p_{\boldsymbol{\theta}}(\mathbf{x}^\textup{miss} | \mathbf{x}^\textup{obs})$. Actually, we will see that a simple modification of this scheme allows to generate \emph{exactly} according to the conditional distribution.

\subsection{Metropolis-within-Gibbs sampling}

At each step of the chain, rather than generating codes according to the approximate posterior distribution, we may \emph{use this approximation as a proposal within a Metropolis-Hastings algorithm} \citep{Metropolis1953,Hastings1970}, using the fact that we have access to the unnormalised posterior density of the latent codes.


Specifically, at each step, we will generate a new code $\tilde{\mathbf{z}}_t$ as a proposal using the approximate posterior $\Psi(\mathbf{z}|g_{\boldsymbol{\gamma}}(\mathbf{x}^\textup{obs},\hat{\mathbf{x}}^\textup{miss}_{t-1}))$. This proposal is kept as a valid code with acceptance probability 
\begin{equation}
\rho_t = \frac{\Phi (\mathbf{x}^\textup{obs},\hat{\mathbf{x}}^\textup{miss}_{t-1}| f_{\boldsymbol{\theta}}(\tilde{\mathbf{z}}_t ) ) p(\tilde{\mathbf{z}}_t)}{\Phi (\mathbf{x}^\textup{obs},\hat{\mathbf{x}}^\textup{miss}_{t-1}| f_{\boldsymbol{\theta}}(\mathbf{z}_{t-1} ) )  p(\mathbf{z}_{t-1})}\frac{ \Psi(\mathbf{z}_{t-1}|g_{\boldsymbol{\gamma}}(\mathbf{x}^\textup{obs},\hat{\mathbf{x}}^\textup{miss}_{t-1}))}{ \Psi(\tilde{\mathbf{z}}_{t}|g_{\boldsymbol{\gamma}}(\mathbf{x}^\textup{obs},\hat{\mathbf{x}}^\textup{miss}_{t-1}))}.
\end{equation} This probability corresponds to a \emph{ratio of importance ratios}, and is equal to one when the posterior approximation is perfect. This code-generating scheme exactly corresponds to performing a single iteration of an \emph{independent Metropolis-Hastings algorithm}. With the obtained code $\mathbf{z}_t$, we can now generate a new imputation using the exact conditional $ \Phi(\mathbf{x}^\textup{miss}|\mathbf{x}^\textup{obs}, f_{\boldsymbol{\theta}}(\mathbf{z}_{t}))$. The obtained algorithm, detailed in Algorithm \ref{alg:mwg}, is a particular instance of a Metropolis-within-Gibbs algorithm. Actually, it exactly corresponds to the algorithm described by \citet[Section 4.4]{gelman1993}, and is \emph{ensured to asymptotically produce samples from the true conditional ditribution} $p_{\boldsymbol{\theta}}(\mathbf{x}^\textup{miss} | \mathbf{x}^\textup{obs})$, even if the variational approximation is imperfect. Note that when the variational approximation is perfect, all proposals are accepted and the algorithm exactly reduces to Gibbs sampling.

	\begin{wrapfigure}{R}{0.57\textwidth}
	\vspace{-9mm}
	\begin{minipage}{0.57\textwidth}
		\begin{algorithm}[H]
			\caption{Metropolis-within-Gibbs sampler for missing data imputation using a trained VAE}
			\label{alg:mwg}
			\begin{algorithmic}
				\STATE {\bfseries Inputs:} Observed data $\mathbf{x}^\textup{obs}$, trained VAE $(f_{\boldsymbol{\theta}},g_{\boldsymbol{\gamma}})$, number of iterations $T$
				\\
				\STATE \textbf{Initialise} $(\mathbf{z}_0,\hat{\mathbf{x}}^\textup{miss}_{0})$ 
				\FOR{$t=1$ {\bfseries to} $T$}
				\STATE $\tilde{\mathbf{z}}_t \sim \Psi(\mathbf{z}|g_{\boldsymbol{\gamma}}(\mathbf{x}^\textup{obs},\hat{\mathbf{x}}^\textup{miss}_{t-1}))$
				\STATE $\rho_t= \frac{\Phi (\mathbf{x}^\textup{obs},\hat{\mathbf{x}}^\textup{miss}_{t-1}| f_{\boldsymbol{\theta}}(\tilde{\mathbf{z}}_t ) ) p(\tilde{\mathbf{z}}_t)}{\Phi (\mathbf{x}^\textup{obs},\hat{\mathbf{x}}^\textup{miss}_{t-1}| f_{\boldsymbol{\theta}}(\mathbf{z}_{t-1} ) )  p(\mathbf{z}_{t-1})}\frac{ \Psi(\mathbf{z}_{t-1}|g_{\boldsymbol{\gamma}}(\mathbf{x}^\textup{obs},\hat{\mathbf{x}}^\textup{miss}_{t-1}))}{ \Psi(\tilde{\mathbf{z}}_{t}|g_{\boldsymbol{\gamma}}(\mathbf{x}^\textup{obs},\hat{\mathbf{x}}^\textup{miss}_{t-1}))}$
				\STATE $\mathbf{z}_t=\left\{
				\begin{array}{lll}
				\tilde{\mathbf{z}}_t  & \textup{with probability } \rho_t \\
				\mathbf{z}_{t-1} & \textup{with probability } 1-\rho_t
				\end{array}
				\right.$
				\STATE  $ \hat{\mathbf{x}}^\textup{miss}_{t} \sim \Phi(\mathbf{x}^\textup{miss}|\mathbf{x}^\textup{obs}, f_{\boldsymbol{\theta}}(\mathbf{z}_{t}))$
				\ENDFOR
			\end{algorithmic}
		\end{algorithm}
	\end{minipage}
	\vspace{-4mm}
\end{wrapfigure}
The theoretical superiority of the Metropolis-within-Gibbs scheme compared to the pseudo-Gibbs sampler comes \emph{with almost no additional computational cost}. Indeed, all the quantities that need to be computed in order to compute the acceptance probability need also to be computed within the pseudo-Gibbs scheme -- except for prior evaluations, which are assumed to be computationally negligible. However, a poor initialisation of the missing values might lead to a lot of rejections at the beginning of the chain, and to slow convergence. A good initialisation heuristic is to perform a few pseudo-Gibbs iterations at first in order to begin with a sensible imputation. Note eventually that, similarly to the pseudo-Gibbs sampler, our Metropolis-within-Gibbs scheme can be extended to many other variational approximations -- like normalising flows \citep{rezende2015,kingma2016} -- in a straightforward manner.

\section{Empirical results}

In this section, we investigate the empirical realisations of our theoretical findings on DLVMs. 

\paragraph{Implementation.} Some computational choices are common for all experiments: the prior distribution is a standard Gaussian distribution; the chosen variational family $(\Psi(\cdot | \boldsymbol{\kappa}))_{\boldsymbol{\kappa} \in K}$ is the family of Gaussian distributions with ``diagonal + rank-1'' covariance (as in \citealp[Section 4.3]{rezende2014}); stochastic gradients of the ELBO are computed via the \emph{path derivative estimate} of \citet{roeder2017}; the Adam optimiser \citep{kingma2014adam}  is used with a learning rate of $10^{-4}$ and mini-batches of size 10; neural network are initialised following the heuristics of \citet{glorot2010}; sampling for variational autoencoders is performed via the Distributions module of TensorFlow \citep{dillon2017}.
For the Frey faces data set, we used a Gaussian DLVM together with the parametrisation presented in Section 2.1 (with $d=5$ and $h=200$). The architecture of the inference network follows the same parametrisation. Constrained finite Gaussian mixtures were fit using the scikit-learn package \citep{pedregosa2011}. Regarding the data sets considered in the missing data experiments, we used the architecture of \citet{rezende2014}, with $200$ hidden units and an intrinsic dimension of $50$. 

	\begin{wrapfigure}{R}{0.57\textwidth}
			\vspace{-15mm}
	\begin{center}
		\includegraphics[width=0.57\columnwidth]{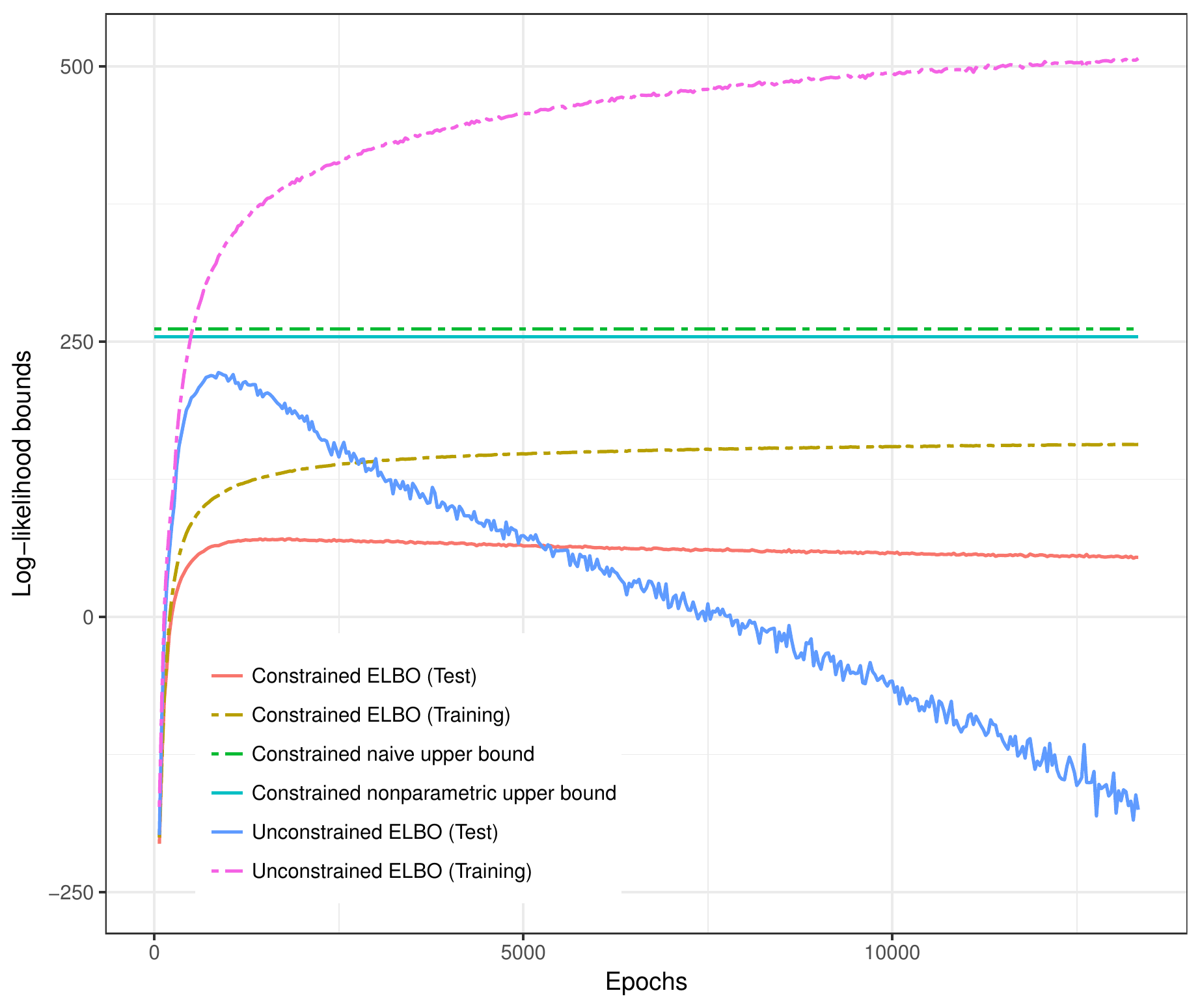}
			\vspace{-6mm}
		\caption{Likelihood blow-up for the Frey Faces data. The unconstrained ELBO appears to diverge, while finding poorer and poorer models.}
		\label{bounds}
	\end{center}
			\vspace{-5mm}
	\end{wrapfigure}


\subsection{Witnessing likelihood blow-up}

To investigate if the unboundedness of the likelihood of a DLVM with Gaussian outputs has actual concrete consequences for variational inference, we train two DLVMs on the Frey faces data set: one with no constraints, and one with the constraint of Proposition \ref{prop:constraints} (with $\xi = 2^{-4}$). The results are presented in Fig. \ref{bounds}. One can notice that the unconstrained DLVM finds models with very high likelihood but very poor generalisation performance. This confirms that the unboundedness of the likelihood is not a merely theoretical concern. We also display the two upper bounds of the likelihood. The nonparametric bound offers a slight but significant improvement over the naive upper bound. On this example, using the nonparametric upper bound as an early stopping criterion for the unconstrained ELBO appears to provide a good regularisation scheme --  that perform better than the covariance constraints on this data set. This illustrates the potential practical usefulness of the connection that we drew between DLVMs and nonparametric mixtures.

\subsection{Comparing the pseudo-Gibbs and Metropolis-within-Gibbs samplers}

We compare the two samplers for single imputation of the test sets of three data sets: Caltech 101 Silhouettes and statically binarised versions of MNIST and OMNIGLOT. We consider two missing data scenarios: a first one with pixels missing uniformly at random (the fractions of missing data considered are $40\%,50\%,60\%,70\%$, and $80\%$) and one where the top or bottom half of the pixels was removed. Both samplers use the same trained VAE and perform the same number of iterations ($2.000$ for the first scenario, and $10.000$ for the second). Note that the convergence can be monitored using a validation set of complete data. The first $20$ iterations of the Metropolis-within-Gibbs sampler are actually pseudo-Gibbs iterations, in order to avoid having too many early rejections. The chains converge much faster for the missing at random (MAR) situation than for the top/bottom missing scenario. This is probably due to the fact that the conditional distribution of the missing half of an image is highly multimodal. The results are displayed on Fig. \ref{fig:imput} and in the supplementary material. The Metropolis-within-Gibbs sampler consistently outperforms the pseudo-Gibbs scheme, especially for the most challenging scenarios where the top/bottom of the image is missing. Moreover, according to Wilcoxon signed-rank tests (see e.g. \citealp[Chapter 5]{dalgaard2008}), the Metropolis-within-Gibbs scheme is significantly more accurate than the pseudo-Gibbs sampler in all scenarios (all $p$-values are provided as supplementary material).
One can see that the pseudo-Gibbs sampler appears to converge quickly to a stationary distribution that gives suboptimal results. The Metropolis-within-Gibbs algorithm converges slower, but to a much more useful distribution.

\begin{figure}
	\begin{center}
		\includegraphics[width=\columnwidth]{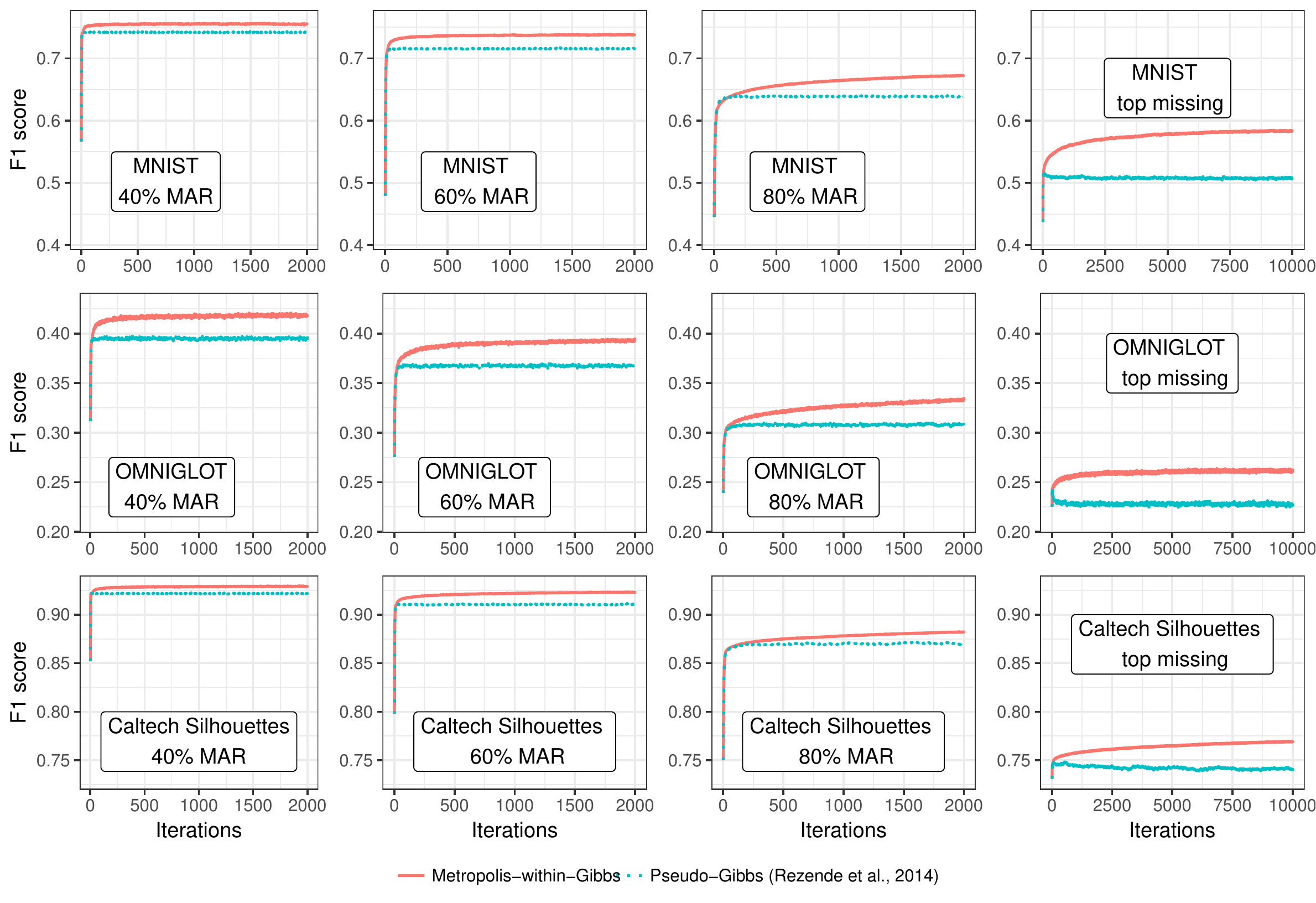}
			\vspace{-6mm}
		\caption{Single imputation results (F1 score between the true and imputed values) for the two Markov chains. Additional results for the bottom missing and the $50\%$ and $70\%$ MAR cases are provided as supplementary material. The more the conditional distribution is challenging (high-dimensional in the MAR cases and highly multimodal in the top/bottom cases), the more the performance gain of our Metropolis-within-Gibbs scheme is important.}
		\label{fig:imput}
	\end{center}
	\vspace{-7mm}
\end{figure}




\section{Conclusion}

Although extremely difficult to compute in practice, the exact likelihood of DLVMs offers several important insights on deep generative modelling. An important research direction for future work is the design of principled regularisation schemes for maximum likelihood estimation.

The objective evaluation of deep generative models remains a widely open question. Missing data imputation is often used as a performance metric for DLVMs (e.g. \citealp{li2016,du2018}). Since both algorithms have essentially the same computational cost, this motivates to replace pseudo-Gibbs sampling by Metropolis-within-Gibbs when evaluating these models. Upon convergence, the samples generated by Metropolis-within-Gibbs do not depend on the inference network, and explicitly depend on the prior, which allows us to evaluate mainly the generative performance of the models. 

We interpreted DLVMs as parsimonious submodels of nonparametric mixture models. While we used this connection to provide upper bounds of the likelihood, many other applications could be derived. In particular, the important body of work regarding consistency of maximum likelihood estimates for nonparametric mixtures (e.g. \citealp{kiefer1956,van2003,chen2017}) could be leveraged to study the asymptotics of DLVMs.



\setcounter{theorem}{0}

\section*{Appendix A. Proof of Theorem 1}

\begin{theorem} For all $i \in \{1,...,n\}$ and $\mathbf{w} \in \mathbb{R}^d \setminus \{0\}$, we have
	$\lim_{k \rightarrow + \infty} \ell \left(\boldsymbol{\theta}^{(i,\mathbf{w})}_{k}\right) = + \infty. $
\end{theorem}

\begin{proof}
	Let $i \in \{1,...,n\}$ and $\mathbf{w} \in \mathbb{R}^d \setminus \{0\}$. To avoid notational overload, we denote $\boldsymbol{\theta}_k = \boldsymbol{\theta}^{(i,\mathbf{w})}_{k}$ in the remainder of this proof. We will show that the contribution $\log p_{\boldsymbol{\theta}_k}(\mathbf{x}_i)$ of the $i$-th observation explodes while all other contributions remain bounded below.
	
	A first useful remark is the fact that, since $\mathbf{z}$ is continuous and has zero mean and $\mathbf{w} \in \mathbb{R}^d \setminus \{0\}$, the univariate random variable $\mathbf{w}^T\mathbf{z}$ is continuous and has zero mean. Therefore $\mathbb{P}(\mathbf{w}^T\mathbf{z}\leq 0)>0$ and $\mathbb{P}(\mathbf{w}^T\mathbf{z} \geq 0)>0$.
	
	Regarding the $i$-th observation, we have for all $k \in \mathbb{N}^*$,
	\begin{align} \label{eq:bound}
	p_{\boldsymbol{\theta}_k}(\mathbf{x}_i) &= \int_{\mathbb{R}^d} \mathcal{N}(\mathbf{x}_i|\mathbf{x}_i,{\boldsymbol{\Sigma}_{\boldsymbol{\theta}_k}}(\mathbf{z})) p(\mathbf{z})d\mathbf{z} \\ &\geq \int_{\mathbf{w}^T\mathbf{z} \leq 0} \mathcal{N}(\mathbf{x}_i|\mathbf{x}_i,{\boldsymbol{\Sigma}_{\boldsymbol{\theta}_k}}(\mathbf{z})) p(\mathbf{z})d\mathbf{z} \\
	&\geq \int_{\mathbf{w}^T\mathbf{z} \leq 0}  |2 \pi {\boldsymbol{\Sigma}_{\boldsymbol{\theta}_k}}( \mathbf{z})  |^{-1/2} p(\mathbf{z})d\mathbf{z}.
	\end{align}
	Let $\mathbf{z} \in \mathbb{R}^d$ such that $\mathbf{w}^T\mathbf{z} \leq 0$. The function $$\varphi: \alpha \mapsto \alpha \tanh \left(\alpha \mathbf{w}^T\mathbf{z}\right) - \alpha,$$ is strictly decreasing on $\mathbb{R}^+$. Indeed, its derivative is equal to
	\begin{equation} \varphi ' (\alpha)= \tanh(\alpha \mathbf{w}^T\mathbf{z}) -1  + (1-\tanh ^2 (\alpha \mathbf{w}^T\mathbf{z})) \alpha \mathbf{w}^T, \end{equation}
	which is strictly negative because the image of the hyperbolic tangent function is $]-1,1[$. Moreover, $$\lim_{\alpha \rightarrow \infty} \alpha \tanh \left(\alpha \mathbf{w}^T\mathbf{z}\right) - \alpha  = - \infty.$$ Therefore, the sequence $ (|2 \pi {\boldsymbol{\Sigma}_{\boldsymbol{\theta}_k}}( \mathbf{z}) |^{-1/2} )_{k\geq 1}$ is strictly increasing and diverges to $+ \infty$ for all $\mathbf{z} \in \mathbb{R}^d$ such that $\mathbf{w}^T\mathbf{z} \leq 0$. Therefore , the monotone convergence theorem combined with the fact that $\mathbb{P}(\mathbf{w}^T\mathbf{z}\leq 0)>0$ insure that the right hand side of \eqref{eq:bound} diverges to $+ \infty$, leading to $p(\mathbf{x}_i|{\boldsymbol{\mu}_{\boldsymbol{\theta}}}_i,{\boldsymbol{\Sigma}_{\boldsymbol{\theta}_k}}) \rightarrow + \infty$.
	
	Regarding the other contributions, let $j \neq i$. Since $\mathbb{P}(\mathbf{w}^T\mathbf{z}>0)>0$, there exists $\varepsilon >0$ such that $\mathbb{P}(\mathbf{w}^T\mathbf{z}\geq \varepsilon)>0$. We have
		\begin{align*}
	p_{\boldsymbol{\theta}_k}(\mathbf{x}_j)  &\geq \int_{\mathbf{w}^T\mathbf{z} \geq \varepsilon } \mathcal{N}(\mathbf{x}_j|\mathbf{x}_i,{\boldsymbol{\Sigma}_{\boldsymbol{\theta}_k}}(\mathbf{z})) p(\mathbf{z})d\mathbf{z} \\ 
	&= \int_{\mathbf{w}^T\mathbf{z} \geq \varepsilon } \frac{ \exp \left(\frac{-||\mathbf{x}_j-\mathbf{x}_i||_2^2}{2 \exp \left( \alpha_k \tanh \left(\alpha_k \mathbf{w}^T\mathbf{z}\right) - \alpha_k \right)}\right) } {\left(2 \pi \exp \left( \alpha_k \tanh \left(\alpha_k \mathbf{w}^T\mathbf{z}\right) - \alpha_k \right) \right)^{p/2} } p(\mathbf{z})d\mathbf{z} \\
	&= \frac{1}{(2\pi)^{p/2}} \int_{\mathbf{w}^T\mathbf{z} \geq \varepsilon }  \exp \left(\frac{-||\mathbf{x}_j-\mathbf{x}_i||_2^2}{2 \exp \left( \alpha_k \tanh \left(\alpha_k \mathbf{w}^T\mathbf{z}\right) - \alpha_k \right)}\right)  p(\mathbf{z})d\mathbf{z} \\
	&\geq \frac{1}{(2\pi)^{p/2}}  \exp \left(\frac{-||\mathbf{x}_j-\mathbf{x}_i||_2^2}{2 \exp \left( \alpha_k \tanh \left(\alpha_k \varepsilon \right) - \alpha_k \right)}\right)  \mathbb{P}(\mathbf{w}^T\mathbf{z} \geq \varepsilon ),
	\end{align*}
	and, since $\lim_{\alpha \rightarrow \infty}  \alpha \tanh \left(\alpha \varepsilon \right) - \alpha = 0$, we will have
	$$\limsup_{k \rightarrow \infty} p_{\boldsymbol{\theta}_k}(\mathbf{x}_j)  \geq  \frac{ \mathbb{P}(\mathbf{w}^T\mathbf{z})}{(2\pi)^{p/2}}  \exp \left(\frac{-||\mathbf{x}_j-\mathbf{x}_i||_2^2}{2}\right)  \geq \varepsilon ,$$
	therefore
	$$ \limsup_{k \rightarrow \infty} p_{\boldsymbol{\theta}_k}(\mathbf{x}_j)  > 0.$$
	
	By combining all contributions, we end up with $\lim_{k \rightarrow + \infty} \ell(\boldsymbol{\theta}_k) = + \infty$.
\end{proof}

\section*{Appendix B. Proof of Theorem 2}

\begin{theorem}
	\label{th:nonpara}
	Assume that $(\Phi (\cdot|\boldsymbol{\eta}))_{\boldsymbol{\eta}\in H}$ is the family of multivariate Bernoulli distributions or the family of Gaussian distributions with the spectral constraint of Proposition \ref{prop:constraints}. The likelihood of the corresponding nonparametric mixture model is maximised for a finite mixture model of $k \leq n$ distributions from the family $(\Phi (\cdot|\boldsymbol{\eta}))_{\boldsymbol{\eta}\in H}$.
\end{theorem}

\begin{proof}
	Let us assume that there are $L \leq n$ distinct observations $\tilde{\mathbf x}_1,...,\tilde{\mathbf x}_{L}$. Following \citet{lindsay1983}, let us consider the \emph{trace of the likelihood curve}
	$$\Gamma = \{  \left( \Phi (\tilde{\mathbf x}_l | \boldsymbol{\eta} )\right)_{l \leq L} | \boldsymbol{\eta} \in H \}.$$
	We will use the following theorem, which describes maximum likelihood estimators of nonparametric mixtures under some topological conditions.
	
	\begin{theorem*}[\textbf{\citealp[Theorem 3.1]{lindsay1983}}]
		If $\Gamma$ is compact, then the likelihood of the corresponding nonparametric mixture model is maximised for a finite mixture model of $k \leq n$ distributions from the family $(\Phi (\cdot|\boldsymbol{\eta}))_{\boldsymbol{\eta}\in H}$.
	\end{theorem*}
	
	This theorem allows us to treat both cases:
	\paragraph{Bernoulli outputs.} Since $H = [0,1]^p$ is compact and the function $ \boldsymbol{\eta} \mapsto\left(  \Phi (\tilde{\mathbf x}_l  | \boldsymbol{\eta} )\right)_{l \leq L} $ is continuous, $\Gamma$ is compact and Lindsay's theorem can be directly applied.
	\paragraph{Gaussian outputs.} The parameter space $H$ is not compact in this case, but we can get around this problem using a compactification argument similar to the one of \citet{van1992}. Consider the Alexandroff compactification $H\cup \{\infty \}$ of the parameter space \citep[p. 150]{kelley1955}. Because of the definition of $H$, we have $\lim_{\boldsymbol{\eta} \rightarrow \infty}  \Phi (\tilde{\mathbf x}_l  | \boldsymbol{\eta} ) = 0$ for all $l \in \{1,...,L\}$. Therefore, we can continuously extend the function $ \boldsymbol{\eta} \mapsto\left(  \Phi (\tilde{\mathbf x}_l  | \boldsymbol{\eta} ))\right)_{l \leq L} $ from $H$ to $H\cup \{\infty \}$ using the conventions $\Phi(\tilde{\mathbf x}_l  | \infty ) = 0$ for all $l \in \{1,...,L\}$. The space $\{  \left( \Phi (\tilde{\mathbf x}_l | \boldsymbol{\eta} )\right)_{l \leq L} | \boldsymbol{\eta} \in H \cup \{\infty \} \}$ is therefore compact, and we can use Lindsay's theorem to deduce that the nonparametric maximum likelihood estimator for the compactified parameter space $H\cup \{\infty \}$ is a finite mixture of distributions from $(\Phi (\cdot|\boldsymbol{\eta}))_{\boldsymbol{\eta}\in H\cup \{\infty \}}$. However, this finite mixture can only contain elements from $(\Phi (\cdot|\boldsymbol{\eta}))_{\boldsymbol{\eta}\in H}$. Indeed, if any mixture component were associated with $\boldsymbol{\eta} = \infty$, the likelihood could be improved by emptying said component. Therefore, the maximum likelihood estimator found using the compactified space is also the maximum likelihood estimator of the original space, which allows us to conclude.
\end{proof}

\section*{Appendix C. Proof of Theorem 3}

\begin{theorem}
	Assume that \begin{enumerate}
		\item $(\Phi (\cdot|\boldsymbol{\eta}))_{\boldsymbol{\eta}\in H}$ is the family of multivariate Bernoulli distributions or the family of Gaussian distributions with the spectral constraint of Proposition \ref{prop:constraints}.
		\item The decoder has universal approximation abilities : for any compact $K \in  \mathbb{R}^d$ and continuous function $f: K \rightarrow H$, for all $\varepsilon >0$, there exists $\boldsymbol{\theta}$ such that $||f - f_{\boldsymbol{\theta}} ||_\infty < \varepsilon$.
	\end{enumerate}
	Then, for all $\varepsilon >0$, there exists $\boldsymbol{\theta}$ such that
	$$    \ell(\hat{G}) \geq \ell(\boldsymbol{\theta}) \geq \ell(\hat{G}) - \varepsilon  .$$ 
	
\end{theorem}
\begin{proof}
	The left hand side of the inequality directly comes from the first assumption and Theorem \ref{th:nonpara}. We will now prove the right hand side.
	
	Let $\varepsilon > 0$. Since the logarithm is continuous, there exists $\delta > 0$ such that
	\begin{equation}
	\label{eq:delta}
	0\leq u \leq \delta \implies \sum_{i=1}^n \log(p_{\hat{G}}(\mathbf{x}_i) - u ) \geq \ell(\hat{G}) - \varepsilon.
	\end{equation}
	We will show that $p_{\hat{G}}(\mathbf{x}_i)$ and $p_{\boldsymbol{\theta}}(\mathbf{x}_i)$ can be made closer than $\delta$ for all $i\leq n$.
	
	Theorem \ref{th:nonpara} insures that $\hat{G}$ is a finite mixture, so there exist some $(\pi_1,...,\pi_K) \in \Delta_K$ and $\boldsymbol{\eta}_1,...,\boldsymbol{\eta}_K \in H$ such that
	$$\forall \mathbf{x} \in \mathcal{X}, \; p_{\hat{G}}(\mathbf{x}) = \sum_{k=1}^K \pi_k \Phi(\mathbf{x}|\boldsymbol{\eta}_k).$$
	We will partition the domain of integration of the latent variable into several parts:
	\begin{itemize}
		\item an infinite part of very small prior mass,
		\item a compact set of high prior mass over which we can apply the universal approximation property.
	\end{itemize}
	The compact set will be divided itself into $2K$ parts: one part for each mixture component and $K-1$ small parts to insure the continuity of the decoder. More specifically, let $e\in ]0,\min \{\pi_1,...,\pi_K\}[$, let $F$ be the (prior) cumulative distribution function of $||\mathbf{z}||_2$, and let:
	\begin{itemize}
		\item $ \alpha_1 = 0,$
		\item $\forall k \in \{2,...,K\} , \; \alpha_k = F^{-1}(\sum_{l=1}^{k-1}\pi_l),$
		\item $\alpha_{k+1}= + \infty,$
		\item $\forall k \in \{1,...,K\} , \; \beta_k = F^{-1}(\sum_{l=1}^{k}\pi_l - e).$
	\end{itemize}
	Let $C = \{\mathbf{z}\in \mathbb{R}^d\; | \;||\mathbf{z}||_2 \leq \beta_K  \}$. Consider a continuous function $f: C \rightarrow H$ such that 
	$$\forall \mathbf{z} \in K,  f(\mathbf{z}) = \begin{cases} 
	\boldsymbol{\eta}_1 & \textup{if }\alpha_1 \leq ||\mathbf{z}||_2 \leq \beta_1  \\
	\boldsymbol{\eta}_2& \textup{if } \alpha_2 \leq ||\mathbf{z}||_2 \leq \beta_2 \\
	... \\
	\boldsymbol{\eta}_K & \textup{if } \alpha_K \leq ||\mathbf{z}||_2 \leq \beta_K.
	\end{cases}
	$$
	For example, $f$ can be built using the Tietze extension theorem (see e.g. \citealp[p. 242]{kelley1955}) together with the conditions of the above formula. According to the universal approximation property, we can find decoders arbitrarily close to $f$. How close do they need to be? Invoking the continuity of the functions $(\boldsymbol{\eta} \rightarrow \Phi (\mathbf{x}_i|\boldsymbol{\eta}))_{i \leq n}$ in $\boldsymbol{\eta}_1,...,\boldsymbol{\eta}_k$, there exists $\Delta >0$ such that
	$$||\boldsymbol{\eta}_k - \boldsymbol{\eta} ||_{\infty}\leq \Delta \implies |\Phi (\mathbf{x}_i|\boldsymbol{\eta})-\Phi (\mathbf{x}_i|\boldsymbol{\eta}_k)| \leq e.$$
	We will consider now a decoder $f_{\boldsymbol{\theta}}$ such that $|| f - f_{\boldsymbol{\theta}}||_{\infty} \leq \Delta$.

	We can write, for all $i\leq n$,
	\begin{align*}
	p_{\boldsymbol{\theta}}(\mathbf{x}_i) &= \sum_{l=1}^K \left( \int_{\alpha_l \leq ||\mathbf{z}||_2 \leq \beta_l} \Phi (\mathbf{x}_i|f_{\boldsymbol{\theta}}(\mathbf{z}))p(\mathbf{z})d\mathbf{z} + \int_{\beta_l \leq ||\mathbf{z}||_2 \leq \alpha_{l+1}} \Phi (\mathbf{x}_i|f_{\boldsymbol{\theta}}(\mathbf{z}))p(\mathbf{z})d\mathbf{z} \right)  \\
	&\geq \sum_{l=1}^K  \int_{\alpha_l \leq ||\mathbf{z}||_2 \leq \beta_l} \Phi (\mathbf{x}_i|f_{\boldsymbol{\theta}}(\mathbf{z}))p(\mathbf{z})d\mathbf{z}.
	\end{align*}
	Using the fact that $|| f - f_{\boldsymbol{\theta}}||_{\infty} \leq \Delta$, we have therefore, for all $i \leq n$,
	
	\begin{align*}
	p_{\boldsymbol{\theta}}(\mathbf{x}_i) &\geq \sum_{l=1}^K  \int_{\alpha_l \leq ||\mathbf{z}||_2 \leq \beta_l} (\Phi (\mathbf{x}_i|\boldsymbol{\eta}_k)-e)p(\mathbf{z})d\mathbf{z} \\
	&=  \sum_{l=1}^K (\Phi (\mathbf{x}_i|\boldsymbol{\eta}_k)-e)(\pi_k - e) \\
	&=  \sum_{l=1}^K \left(\pi_k \Phi (\mathbf{x}_i|\boldsymbol{\eta}_k) + e^2 -e(\pi_k + \Phi (\mathbf{x}_i|\boldsymbol{\eta}_k)) \right)\\
	&\geq \sum_{l=1}^K\pi_k \Phi (\mathbf{x}_i|\boldsymbol{\eta}_k) -e \left(\sum_{l=1}^K \Phi (\mathbf{x}_i|\boldsymbol{\eta}_k) +1\right).
	\end{align*}
	Now, if we take $e$ such that $$e \left(\sum_{l=1}^K \Phi (\mathbf{x}_i|\boldsymbol{\eta}_k) +1\right)\leq \delta,$$ for all $i \leq n$, we end up with
	$$p_{\boldsymbol{\theta}}(\mathbf{x}_i) \geq p_{\hat{G}}(\mathbf{x}_i) - \delta.$$
	Using \eqref{eq:delta}, this eventually leads to $\ell(\boldsymbol{\theta}) \geq \ell(\hat{G}) - \varepsilon$.
\end{proof}

\section*{Appendix G. Additional imputation experiments}

\begin{figure}[h]
	\begin{center}
		\includegraphics[width=\columnwidth]{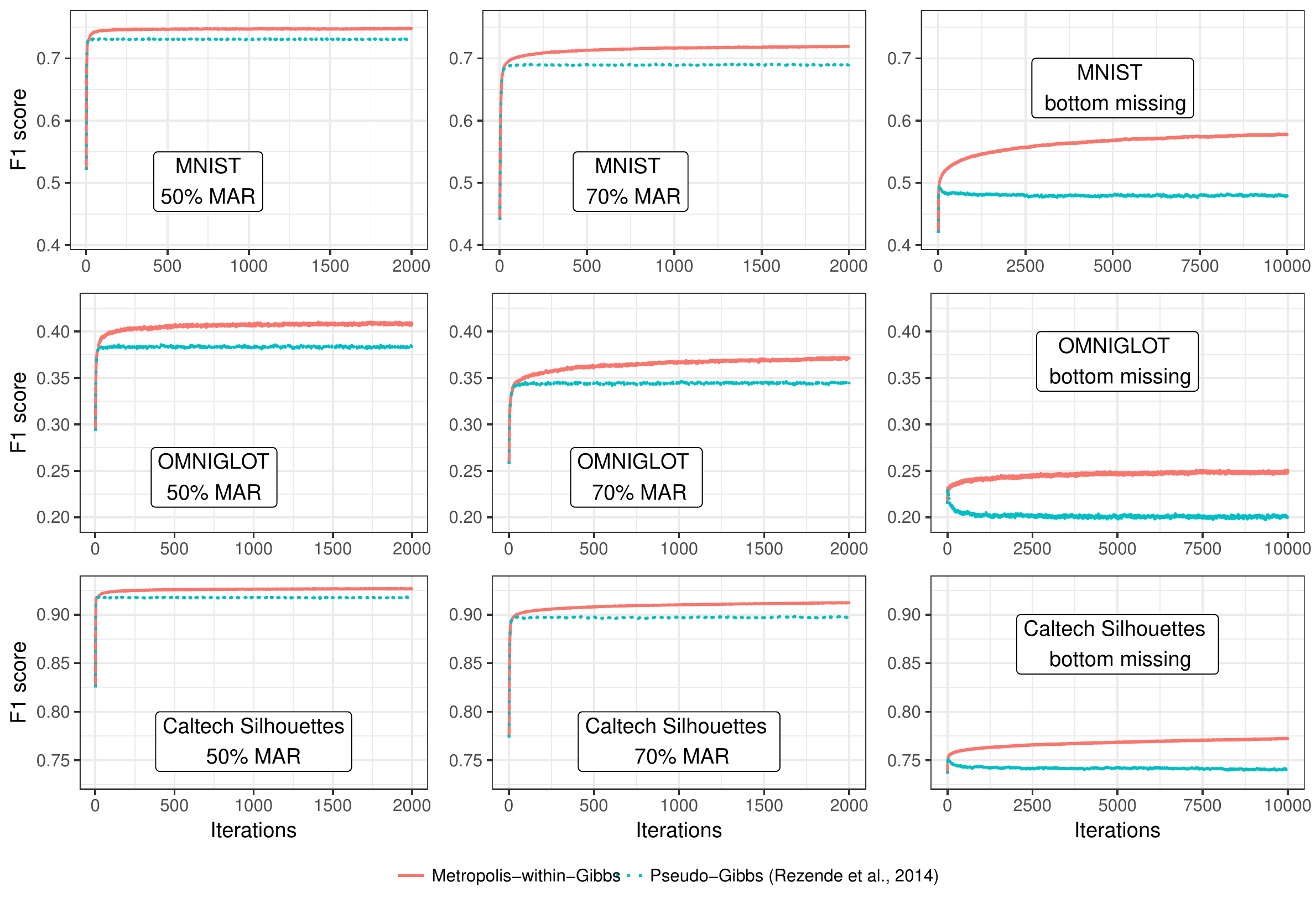}
		\caption{Single imputation results (F1 score between the true and imputed values) for the two Markov chains.}
		\label{fig:imput}
	\end{center}
\end{figure}

We provide below all $p$-values obtained for Wilcoxon signed-rank tests (see e.g. \citealp[Chapter 5]{dalgaard2008}) based on the accuracy of the two imputation schemes over all images in the test sets. The Metropolis-within-Gibbs sampler always significantly outperforms the pseudo-Gibbs scheme.

\begin{table}[h]
	\centering
	\small
	\label{my-label}
	\begin{tabular}{lccccccc}
		& \multicolumn{5}{c}{\textit{Missing at random}}                      & \textit{Top}         & \textit{Bottom}      \\
		& $40\%$      & $50\%$      & $60\%$      & $70\%$      & $80\%$      & \multicolumn{1}{l}{} & \multicolumn{1}{l}{} \\ \hline
		MNIST               & $<10^{-15}$ & $<10^{-15}$ & $<10^{-15}$ & $<10^{-15}$ & $<10^{-15}$ & $<10^{-15}$          & $<10^{-15}$          \\
		OMNIGLOT            & $<10^{-15}$ & $<10^{-15}$ & $<10^{-15}$ & $<10^{-15}$ & $<10^{-15}$ & $0.01$               & $<10^{-15}$          \\
		Caltech Silhouettes & $<10^{-15}$ & $<10^{-15}$ & $<10^{-15}$ & $<10^{-15}$ & $<10^{-15}$ & $<10^{-15}$          & $<10^{-15}$         
	\end{tabular}
	\caption{$p$-values obtained for Wilcoxon signed-rank tests on the accuracy of the two imputation schemes.}
\end{table}

\bibliography{biblio}
\bibliographystyle{plainnat}

\end{document}